\documentclass[10pt,twocolumn,letterpaper]{article}

\usepackage{cvpr}
\usepackage{times}
\usepackage{epsfig}
\usepackage{graphicx}
\usepackage{amsmath}
\usepackage{amssymb}

\usepackage{authblk} 

\usepackage[pagebackref=true,breaklinks=true,letterpaper=true,colorlinks,bookmarks=false]{hyperref}

\usepackage{multirow}
\usepackage{tabularx} 

\newcolumntype{L}[1]{>{\hsize=#1\hsize\raggedright\arraybackslash}X}%
\newcolumntype{R}[1]{>{\hsize=#1\hsize\raggedleft\arraybackslash}X}%
\newcolumntype{C}[1]{>{\hsize=#1\hsize\centering\arraybackslash}X}%

\usepackage{algorithm}
\usepackage{algorithmic}

\usepackage{textcomp}

\usepackage{amsthm}
\theoremstyle{definition}
\newtheorem{definition}{Definition}[section]

\newtheorem{theorem}{Theorem}

\usepackage{soul} 
\usepackage{enumitem} 

\usepackage[caption=false,font=footnotesize,labelfont=sf,textfont=sf]{subfig}

\usepackage[percent]{overpic}

\renewcommand{\hl}[1]{#1}

\newcommand{\cI}{\mathcal{I}}

\newcommand{\cX}{\mathcal{X}}
\newcommand{\cY}{\mathcal{Y}}
\newcommand{\cC}{\mathcal{C}}
\newcommand{\cV}{\mathcal{V}}

\newcommand{\bx}{\mathbf{x}}
\newcommand{\by}{\mathbf{y}}

\newcommand{\col}{f}
\def\mcx{PMC}

\graphicspath{{fig/}}

\cvprfinalcopy 

\begin{document}

	\title{A Practical Maximum Clique Algorithm for Matching with Pairwise Constraints}

\author[1]{\'{A}lvaro Parra}
\author[1]{Tat-Jun Chin}
\author[1]{Frank Neumann}
\author[2]{Tobias Friedrich}
\author[2]{Maximilian Katzmann}

\affil[1]{School of Computer Science, The University of Adelaide}
\affil[2]{Hasso Plattner Institute, University of Potsdam}


	\maketitle

	\begin{abstract}
		A popular paradigm for 3D point cloud registration is by extracting 3D keypoint correspondences, then estimating the registration function from the correspondences using a robust algorithm. However, many existing 3D keypoint techniques tend to produce large proportions of erroneous correspondences or outliers, which significantly increases the cost of robust estimation. An alternative approach is to directly search for the subset of correspondences that are pairwise consistent, without optimising the registration function. This gives rise to the combinatorial problem of matching with pairwise constraints. In this paper, we propose a very efficient maximum clique algorithm to solve matching with pairwise constraints. Our technique combines tree searching with efficient bounding and pruning based on graph colouring. We demonstrate that, despite the theoretical intractability, many real problem instances can be solved exactly and quickly (seconds to minutes) with our algorithm, which makes our approach an excellent alternative to standard robust techniques for 3D registration\footnote{Code and demo program are available in the supplementary material.}.
	\end{abstract}

	\section{Introduction}

	The registration of discrete 3D point sets or point clouds is a recurrent task in computer vision. Often, the point clouds correspond to objects or surfaces that were acquired from the environment using a 3D scanner. Registering point clouds allows to identify the parts in the point clouds that are ``similar". Thus, point cloud registration plays a major role in various applications, such as object recognition, robotic navigation, and digital reconstruction.


A popular paradigm for registering 3D point clouds is by extracting 3D keypoint correspondences, then estimating the registration function from the correspondences. Specifically, let $\cX = \lbrace \bx_i \rbrace_{i=1}^n$  and $\cY = \lbrace \by_j \rbrace_{j=1}^m$ be two input point clouds. A 3D keypoint technique~\cite{tombari13} is used to generate a tentative correspondence set $\cC = \lbrace c_k \rbrace^{N}_{k = 1}$, where each $c_k := (\bx_k, \by_{k^\prime})$ associates a point $\bx_k \in \cX$ to a point $\by_{k^\prime} \in \cY$. If there are no false correspondences or outliers, the registration function, i.e., a 6 DoF rigid transformation, can be estimated easily from $\cC$~\cite{horn87}. Usually, however, outliers exist in $\cC$, thus the registration function must be estimated using a robust technique~\cite{meer04} such as RANSAC~\cite{fischler81}.

Many current 3D keypoint techniques are in fact much less accurate than their 2D image counterparts~\cite{mikolajczyk04,mikolajczyk05}, since the irregular sampling densities on discrete point sets reduce the efficacy of local features. On real point clouds, it is common to encounter outlier rates in excess of $95\%$. Such high outlier rates greatly increase the computational cost of robust estimation. It is thus vital to investigate alternative approaches for keypoint-based 3D registration.

Instead of estimating the registration function from $\cC$, we can attempt to find the largest subset of $\cC$ that are \emph{pairwise consistent}, i.e., we aim to solve
\begin{align}\label{eq:pw}
	\begin{aligned}
	&\underset{ \cI \subseteq \{1,\dots,N \}}{\text{maximise}}
	& & \left| \cI \right| \\
	&\text{subject to}
	& & d( c_i, c_j) \leq \epsilon, \; \forall i,j \in \cI,
	\end{aligned}
\end{align}
where the ``distance" between two correspondences $c_i = (\bx_i, \by_{i^\prime} )$  and $c_j = (\bx_j, \by_{j^\prime})$ is given by
\begin{align}\label{eq:pcrd}
	d( c_i, c_j) = \left| \left\| \bx_i - \bx_j\right\|_2 - \left\|\by_{i^\prime} - \by_{j^\prime} \right\|_2 \right|.
\end{align}
We say that $c_i$ and $c_j$ are consistent if $d( c_i, c_j)$ is less than a predetermined threshold $\epsilon$; intuitively, this means that $c_i$ and $c_j$ are agreeable (up to $\epsilon$) to the same rigid transformation. Solving~\eqref{eq:pw} then yields the largest subset of $\cC$ whose elements are all pairwise consistent; see Fig.~\ref{fig:pwexample}.

\begin{figure*}[ht]\centering
	\subfloat[]{\includegraphics[height=3.1cm]{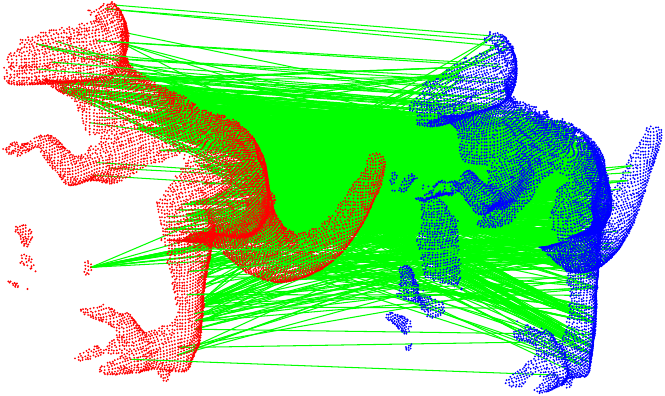}}\hfil
	\subfloat[]{\includegraphics[height=3.1cm]{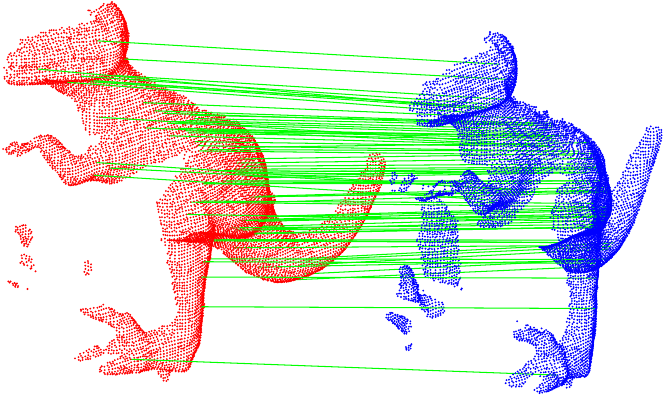}}\hfil
	\subfloat[]{\includegraphics[height=3.1cm]{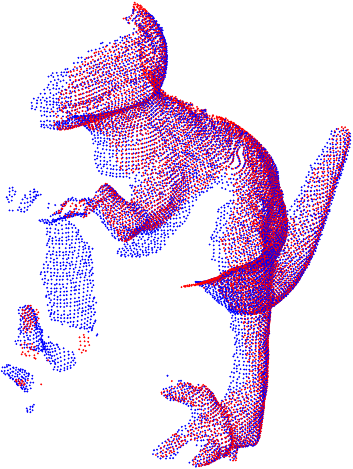}}
	\label{fig:pwexample}
	\caption{Example registration. (a) Input correspondence set $\cC$ of size $N = 2000$. (b) The largest subset of $\cC$ that are pairwise consistent (found in 0.06 seconds by our novel max clique algorithm). (c) The alignment estimated via SVD~\cite{horn87} using the correspondences in (b).}
\end{figure*}

	Problem~\eqref{eq:pw} is a special case of \emph{matching with pairwise constraints}~\cite{bolles_and_cain82,grimson90,tu_et_all_99,enqvist09}\footnote{In the general case, $\cC$ is taken as all $n \times m$ pairings of $\cX$ and $\cY$.}. Observe that~\eqref{eq:pw} does not involve optimising the registration function. Given the solution $\cI^\ast$, the registration function can be estimated from the data indexed by $\cI^\ast$ (using RANSAC~\cite{fischler81} or SVD~\cite{horn87} directly), or the value $|\cI^\ast|$ can directly be taken as the similarity score of shapes $\cX$ and $\cY$. Theoretically, however, problem~\eqref{eq:pw} is intractable (NP-hard) in general; as we will show in Sec.~\ref{sec:graph},~\eqref{eq:pw} is expressible as the classical combinatorial problem of minimum vertex cover. The state-of-the-art algorithm in computer vision based on vertex cover~\cite{enqvist09} can handle only small instances of~\eqref{eq:pw}.

	We note also that, whilst similar, problem~\eqref{eq:pw} is not equivalent to the graph matching or quadratic assignment problem (QAP)~\cite{leordeanu05,zhou13}, since~\eqref{eq:pw} explicitly forbids correspondences that are pairwise inconsistent. Expressing~\eqref{eq:pw} as a QAP would involve a weight matrix with $-\infty$ values, which poses significant difficulties for most QAP solvers.

	\paragraph{Contributions} We show that problem~\eqref{eq:pw} can be solved much more efficiently than previously thought. Specifically, we reformulate~\eqref{eq:pw} as maximum clique, then propose a very efficient algorithm for the problem. Our technique is inspired by a state-of-the-art branch-and-bound (BnB) algorithm~\cite{tomita03}, but with important innovations to the bounding and pruning step based on graph colouring to significantly speed-up the solution - as we will show in Sec.~\ref{sec:results}, our algorithmic improvements are vital to enable the technique on practical-sized input data. Compared to the state-of-the-art method based on vertex cover~\cite{enqvist09}, our method is one order of magnitude faster, and typically requires only seconds to minutes to \emph{globally} solve~\eqref{eq:pw} on realistic input data.

	Our work is the first to illustrate matching with pairwise constraints as a strong alternative to conventional robust fitting procedures for keypoint-based 3D registration. Please see demo program in the supplementary material.

	\section{Graph formulation}\label{sec:graph}

	Let $G= (V, E)$ represent an undirected graph with vertices $V = \{ v_i \}$ and edges $E = \{ (v_i, v_j) \}$.

	\begin{definition}[Adjacency and degree]
		We say that a pair of vertices $v_i$ and $v_j$ of $G$ are adjacent if $(v_i, v_j) \in E$. For each $v_i \in V$, denote the adjacency of $v_i$ as
		\begin{align}
		\Gamma(v_i)=\{ v_j \in V \,|\; (v_i,v_j) \in E\}.
		\end{align}
		Then $|\Gamma(v_i)|$ is called the degree of $v_i$.
	\end{definition}

	\begin{definition}[Consistency graph]
		Given a set of correspondences $\cC$, the consistency graph is constructed as the graph with vertices $V = \cC$ and edges
		\begin{align}\label{eq:e}
		E = \left\{(c_i, c_j) \in \cC \times \cC \mid d( c_i, c_j) \leq \epsilon, \; i \neq j \right\},
		\end{align}
		i.e., two correspondences $c_i$ and $c_j$ are adjacent in the graph if they are pairwise consistent.
	\end{definition}

	\begin{definition}[Inconsistency graph]
		The inconsistency graph is the complement of the consistency graph, i.e., the graph with vertices $V = \cC$ and edges
		\begin{align}
		E = \left\{(c_i, c_j) \in \cC \times \cC \mid d( c_i, c_j) > \epsilon, \; i \neq j \right\}.
		\end{align}
		i.e., two correspondences $c_i$ and $c_j$ are adjacent in the graph if they are pairwise inconsistent.
	\end{definition}



	\begin{definition}[Clique]
		A clique of a graph $G = (V, E)$ is a subgraph of $G$ where every pair of vertices in the subgraph are adjacent. A \emph{maximum clique (MC)} of $G$ is a clique of $G$ with the largest size.
	\end{definition}

	\begin{definition}[Vertex cover]
		A vertex cover of a graph $G = (V, E)$ is a subset of $V$ such that every edge in $E$ is incident with at least one vertex in the subset. The removal of a vertex cover from $G$ leaves an independent set, i.e., a set of vertices with no edges. A \emph{minimum vertex cover (MVC)} of $G$ is a vertex cover of $G$ with the smallest size.
	\end{definition}

	By the above definitions, problem~\eqref{eq:pw} is equivalent to finding the MC of the \emph{consistency} graph constructed from the correspondence set $\cC$. Conversely, the complementary problem to~\eqref{eq:pw} is then finding the MVC of the \emph{inconsistency} graph constructed from $\cC$, i.e., remove the least number of correspondences such that the remaining correspondences are not pairwise inconsistent with each other.

\subsection{MIP solutions}

Both MC and MVC can be expressed as mixed integer programs (MIP). Given input graph $G = (V, E)$, MC can be written as the MIP
	\begin{align}\label{eq:ilpmc}
	\begin{aligned}
	&\text{maximise}
	& & \sum_{i=1}^{|V|} x_i  \\
	&\text{subject to}
	& & x_i + x_j \leq 1, \; \forall (v_i,v_j) \notin E\\
	& & & x_i \in \{0,1\}, \; i = 1 \ldots |V|.
	\end{aligned}
	\end{align}
	The MIP formulation for MVC is
	\begin{align}\label{eq:vc}
	\begin{aligned}
	&\text{minimise}
	& & \sum_{i=1}^{|V|} x_i  \\
	&\text{subject to}
	& & x_i + x_j \geq 1, \; \forall (v_i,v_j) \in E\\
	& & & x_i \in \{0,1\}, \; i = 1 \ldots |V|.
	\end{aligned}
	\end{align}
	Both formulations can then be solved using industry-grade MIP solvers such as IBM CPLEX and Gurobi Optimiser. Since these are mature implementations, we do not further discuss their details; suffice to say that we will compare against the MIP solvers as baselines in the experiments.

	For MVC, Enqvist et al.~\cite{enqvist09} proposed a BnB method where the bounding is conducted using a factor-$2$ approximation technique. Our own experimentation suggests however that their algorithm is generally slower than MIP. In Sec.~\ref{sec:bnbextra}, we propose a novel MC algorithm that is able to significantly outperform the generic MIP solvers.

\section{Maximum clique algorithm}

In this section, we first describe a state-of-the-art MC algorithm called \emph{MCQ}~\cite{tomita03}. As we will show in Sec.~\ref{sec:results}, on realistic input data for point cloud registration, MCQ is still unable to provide good performance. However, MCQ will form the basis for our algorithm to be described in Sec.~\ref{sec:bnbextra}.

\subsection{BnB}\label{sec:simple}


MCQ is a BnB algorithm. Given an input graph $G = (V, E)$, BnB systematically explores the set of cliques of $G$ by building a search tree over the vertices $V$. Before elaborating MCQ, we explain a basic BnB method summarised in Algorithm~\ref{alg:simple}, which represents the basic structure of MCQ.

\begin{algorithm}
\begin{algorithmic}[1]
\REQUIRE A set of candidate vertices $S$.
\STATE \textbf{global variables:} The current clique $R$ and the best clique found so far $R_{best}$.
\STATE \textbf{initialisation:}  $R\leftarrow \emptyset$, $R_{best} \leftarrow \emptyset$.
			\WHILE { $S \neq \emptyset$ }
			\IF {$|R| + |S| \leq|R_{best}|$}
			\RETURN
			\ENDIF
			\STATE $v \leftarrow $ first vertex in $S$.
			\STATE $R \leftarrow R \cup \{v\}$.
			\STATE $S^{\prime} \leftarrow S \cap \Gamma(v)$.
			\IF { $S^{\prime} \neq \emptyset$ }
			\STATE Recursive call with candidate vertices $S^{\prime}$.
			\ELSIF {$|R|>|R_{best}|$}
			\STATE $R_{best} \leftarrow R$.
			\ENDIF
			\STATE $R \leftarrow R \setminus \{v\}$.
			\STATE $S \leftarrow S \setminus \{v\}$. \label{alg:simple:remv}
			\ENDWHILE
			\RETURN $R_{best}$.
	\end{algorithmic}
	\caption{Basic BnB algorithm for MC.}
	\label{alg:simple}
\end{algorithm}

To more intuitively explain Algorithm~\ref{alg:simple}, we use the sample input graph in Fig.~\ref{fig:mcexample}. The first 12 steps of Algorithm~\ref{alg:simple} when applied to the sample graph are shown in Fig.~\ref{fig:tree}.

\begin{figure}[ht]\centering
	\includegraphics[width=.56\linewidth]{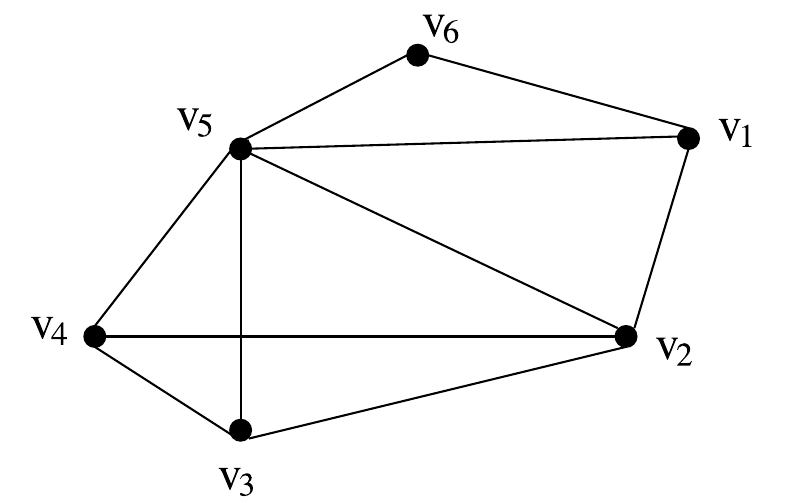}
	\caption{A sample input graph $G = (V, E)$.}
	\label{fig:mcexample}
\end{figure}

\begin{figure*}[ht]\centering
\begin{overpic}[width=\textwidth,grid=false,tics=5]{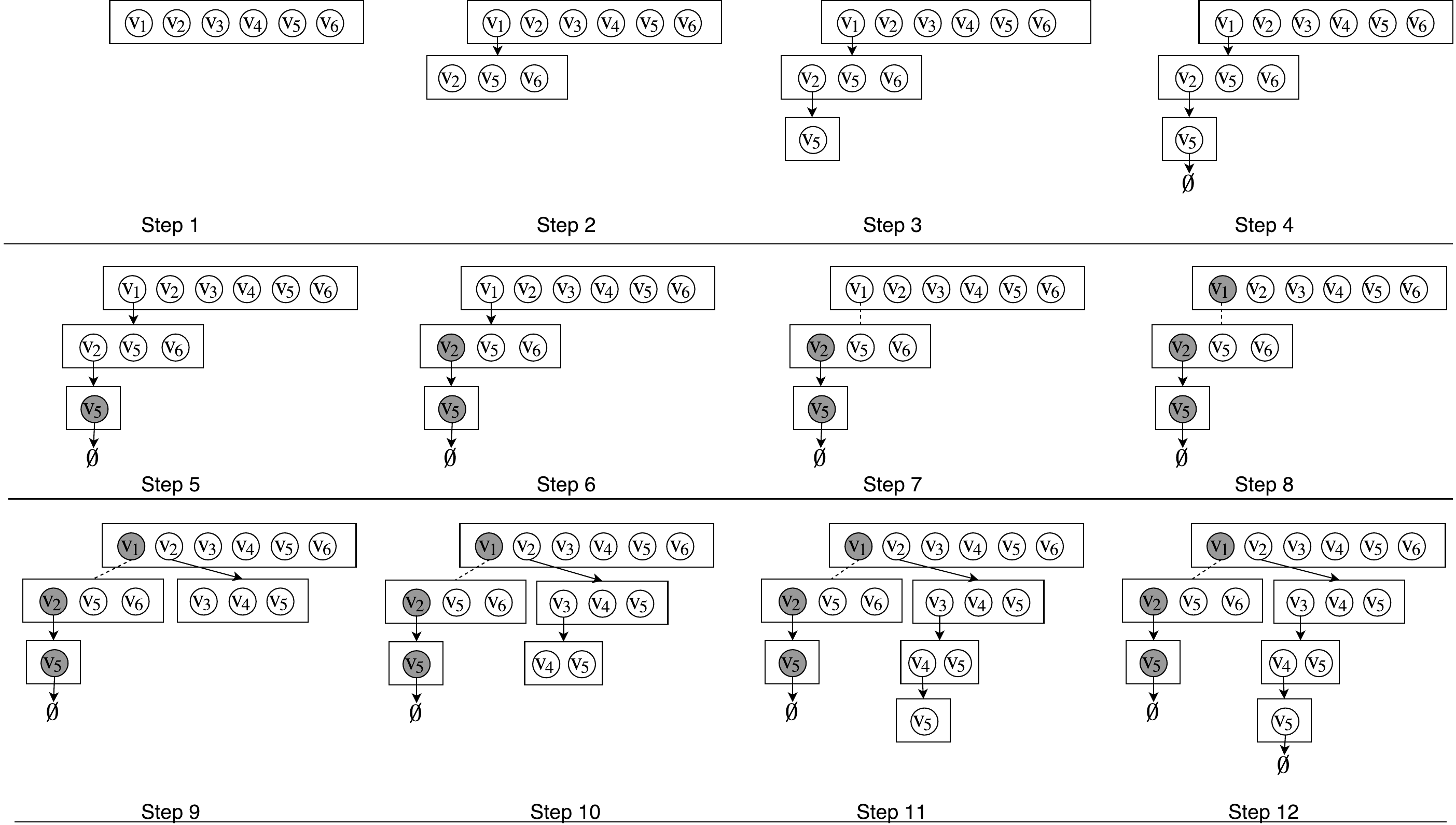}
	\put (0,54.5)  {\small Initial $S$}
	\put (26.5,50.5) {\small $S^{\prime}$}
	\put (51,46.5) {\small $S^{\prime\prime}$}
	\put (77.5,43.5)  {\small $S^{\prime\prime\prime}$}
	\put (84,43) {\small $R_{best} \leftarrow \{v_1, v_2, v_5\}$}
	\put (58.9,29.5) {\small stop expanding since}
	\put (58.9,27.2) {\small $R = \{v_1\}, S^{\prime\prime}=\{v_5, v_6\}$,}
	\put (58.9,25) {\small then, $|R|+|S|\leq |R_{best}|$}
	\put (78,2)  {\small $R_{best} \leftarrow \{v_2, v_3, v_4, v_5\}$}
\end{overpic}
\caption{Progression of the search tree of Algorithm~\ref{alg:simple} when solving MC for the example graph in Fig.~\ref{fig:mcexample}. The first 12 steps are  incrementally generated to show vertex expansions, vertex deletions (grey vertices) and stopping expanding for a node (dashed line). Leaves of the tree correspond to the empty set.}
\label{fig:tree}
\end{figure*}



In the initialisation, the set $S$ containing all vertices in $V$ forms the root node of the BnB search tree; see Step 1 in Fig.~\ref{fig:tree}. The tree is then explored in a depth-first manner, where each branch is initiated by recursively expanding a vertex in $S$. A vertex is expanded by creating a child node $S^\prime$ for $S$ that contains the vertices in $S$ that are adjacent to the vertex to be expanded. Observe in Step 2 in Fig.~\ref{fig:tree} that vertex $v_1$ in the root node $S$ is expanded to yield $S^\prime = \{v_2, v_5, v_6 \}$. The process then continues recursively - e.g., vertex $v_2$ in $S^\prime$ is expanded to yield node $S^{\prime \prime} = \{ v_5 \}$, and $v_5$ in $S^{\prime\prime}$ is expanded to $S^{\prime\prime\prime} = \emptyset$. Note that although $v_2$ is adjacent to other vertices (i.e., $v_3$ and $v_4$) in the graph, $S^{\prime \prime}$ contains $v_5$ only since only $v_5$ exists in the candidate set $S^\prime$ from which $v_2$ was expanded.



The above recursive expansion strategy ensures that a candidate clique $R$ is associated with every node of the search tree ($R=\emptyset$ for the root node) - e.g., $S^\prime$ is associated to $R=\{v_1\}$, $S^{\prime\prime}$ to $R=\{v_1,v_2\}$, and $S^{\prime\prime\prime}$ to $R=\{v_1,v_2,v_5\}$; see Fig.~\ref{fig:tree}.

When an empty node is reached, the largest clique found thus far is recorded as $R_{best}$ - e.g.,  in Step 4 in Fig.~\ref{fig:tree}, the expansion of the single vertex of $S^{\prime\prime}$ creates $S^{\prime\prime\prime}=\emptyset$ and  $R_{best}$ is upgraded to $ R_{best} = \{ v_1, v_2, v_5 \}$.


After a complete branch is expanded for a vertex,  that vertex is removed from the node - e.g., in Step 5 in Fig.~\ref{fig:tree}, $v_5$ is eliminated from $S^{\prime\prime}$; in Step 6, $v_2$ is eliminated from $S^\prime$; and in Step 8, $v_1$ is eliminated from the root node.  This removing step does not compromise optimality since any larger clique containing $R$ and the removed vertex  was already expanded.

A fundamental aspect of BnB algorithms is to prune branches in the search tree that are not promising. To conduct pruning, BnB evaluates an upper bound function of quality achievable at a branch. In Algorithm~\ref{alg:simple}, we simply discard a branch if the depth of the node (equal to $|R|$) plus the number of nodes in the candidate set is not greater than the size of $R_{best}$. In Step 7 in Fig.~\ref{fig:tree}, the algorithm stops expanding vertices in  $S^{\prime\prime}$ since at most a clique of size 3 can be obtained by expanding its vertices (at that stage $|R|=1$, $|S^{\prime\prime}|=2$ and $|R_{best}|=3$).


\subsection{Pruning with graph colouring}

In the ``full version" of MCQ, a more aggressive pruning is conducted using graph colouring. Algorithm~\ref{alg:expand} summarises the procedure.



\begin{algorithm}
\begin{algorithmic}[1]
\REQUIRE A set of candidate vertices $S$, a colouring $\col$.

\STATE \textbf{global variables:} The current clique $R$ and the best clique found so far $R_{best}$.

\STATE \textbf{initialisation:} $R\leftarrow \emptyset$, $R_{best} \leftarrow \emptyset$.

\STATE \textbf{initialisation:} Reorder vertices in $S$ in descending order of degree, i.e., $|\Gamma(v_i)| \geq |\Gamma(v_j)|,\; \forall v_i,v_i\in S $ if $i<j$. \label{alg:mcq:initsort}

\WHILE { $S \neq \emptyset$ }
	\STATE $v \leftarrow $ last vertex in $S$.

	\IF {$|R|+\col(v)\leq|R_{best}|$}
	\RETURN
	\ENDIF

	\STATE $R \leftarrow R \cup \{v\}$.
	\STATE $S^{\prime} \leftarrow S \cap \Gamma(v)$.

	\IF { $S^{\prime} \neq \emptyset$ }
	\STATE Find a colouring $\col^\prime$ of $S^\prime$.  \label{alg:expand:colour}
	\STATE Recursive call with candidate vertices $S^{\prime}$ and colouring $\col^\prime$.


	\ELSIF {$|R|>|R_{best}|$}
	\STATE $R_{best} \leftarrow R$.

	\ENDIF

	\STATE $R \leftarrow R \setminus \{v\}$.

	\STATE $S \leftarrow S \setminus \{v\}$. \label{alg:line:remv}
	\ENDWHILE

	\RETURN $R_{best}$.
\end{algorithmic}
\caption{MCQ algorithm for MC.}
\label{alg:expand}
\end{algorithm}


A colouring of a graph $G=(V,E)$ is a labelling $\col$ of its vertices such that no adjacent vertices have the same colour
\begin{align}
\col(v_i) \neq \col(v_j)\; \text{for all} \; (v_i, v_j) \in E,
\end{align}
and only consecutive numbers starting from 1 are used as colours. The last constraint ensures that the number of colours is equal to the largest colour.

To colour a graph containing a clique $Q$, at least $|Q|$ colours are needed, since each vertex in $Q$ is adjacent to $|Q|-1$ vertices. Hence, a colouring of a subgraph $S$ in the search tree, can be used to obtain an upper bound $u$ of the size of the largest clique of expanding a vertex in $S$.
\begin{align}
	u = |R|+ \max \{ \col(v_i) \; | \; v_i \in S \},
\end{align}
where $R$ is the associated clique to $S$.

To make $u$ a tight upper bound, and hence produce a more aggressive pruning in BnB, we want to find the colouring with the minimum number of colours
\begin{align}
	\min_\col \max_{v_i \in S} \col(v_i). \label{eq:minN}
\end{align}

Finding the minimum colouring~\eqref{eq:minN} is a well known NP-hard problem. MCQ uses a greedy heuristic algorithm for colouring: after an initial sorting of vertices that reduces the average branching factor (Line~\ref{alg:mcq:initsort}), vertices are reordered in increasing colour order such that the rightmost vertex in $S$ has the largest colour. Thus, vertices in $S$ are expanded and removed from right-to-left.


\section{A faster maximum clique algorithm}\label{sec:bnbextra}

\begin{algorithm}
\begin{algorithmic}[1]

\REQUIRE A set of candidate vertices $S$, removed vertices $F$, a colouring $\col$.

\STATE \textbf{global variables:} The current clique $R$ and the best clique found so far $R_{best}$
\STATE \textbf{initialisation:} $R\leftarrow \emptyset$, $R_{best} \leftarrow \emptyset$.

\STATE \textbf{initialisation:} Reorder vertices in $S$ (e.g.,  as in Line~\ref{alg:mcq:initsort} of Algorithm~\ref{alg:expand}). \label{alg:mcx:initsort}


			\STATE Find $C\subseteq S$ as described in Sec.~\ref{sec:bnbextra}.

			\STATE $i \leftarrow |S|$.
			\WHILE { $i>0$ }

			\IF {$|R| + \max_{v_j \in S}\col(v_j)\leq|R_{best}|$} \label{alg:line:pruning}
			\RETURN 
			\ENDIF

			\STATE $v_i \leftarrow $ i-th vertex of $S$.

			\IF {$ v_i \in C$} \label{alg:line:skipping}

			\STATE $R \leftarrow R \cup \{v_i\}$.
			\STATE $S^{\prime} \leftarrow S \cap \Gamma(v_i)$.

			\IF { $S^{\prime} \neq \emptyset$ }

			\STATE $F^\prime \leftarrow F \cap \Gamma(v_i)$.
			\STATE Find a colouring $\col^\prime$ of $S^\prime$.

			\STATE Recursive call with candidate vertices $S^{\prime}$, removed vertices $F^\prime$ and colouring $\col^\prime$.

			\ELSIF {$|R|>|R_{best}|$}
			\STATE $R_{best} \leftarrow R^\prime$.

			\ENDIF
			\STATE $R \leftarrow R \setminus \{v_i\}$.
			\STATE $S \leftarrow S \setminus \{v_i\}$.
			\STATE $F \leftarrow F \cup \{v_i\}$.

\STATE Check colours are consecutive numbers, update if not. \label{alg:mcx:updatecol}

			\ENDIF
			\STATE $i \leftarrow i-1$.
			\ENDWHILE

			\RETURN $R_{best}$.

	\end{algorithmic}
	\caption{PMC algorithm for MC.}
	\label{alg:expand2}
\end{algorithm}

To speed up MCQ, we introduce a novel extra pruning step to avoid exploring multiple branches during the search for the optimal solution; the resultant method is summarised in Algorithm~\ref{alg:expand2}. Our pruning step complements the original pruning technique based on graph colouring. We call the proposed algorithm PMC (Practical Maximum Clique).

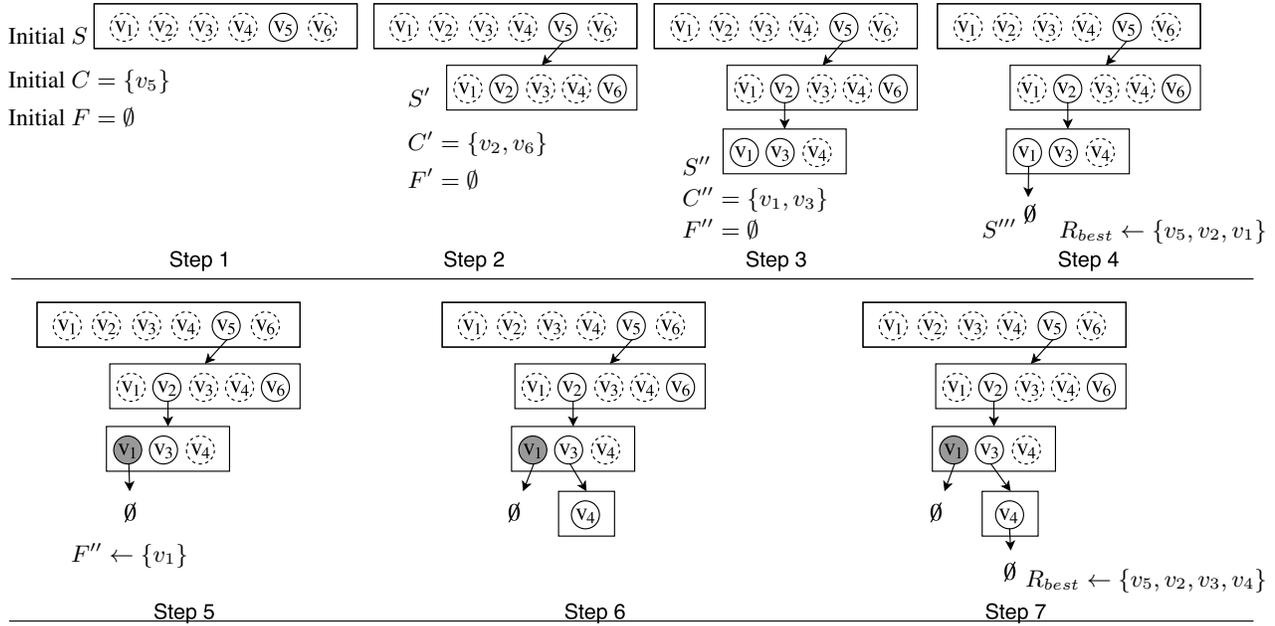
\begin{figure*}
	\centering
	\begin{overpic}[width=.95\textwidth,grid=false,tics=5]{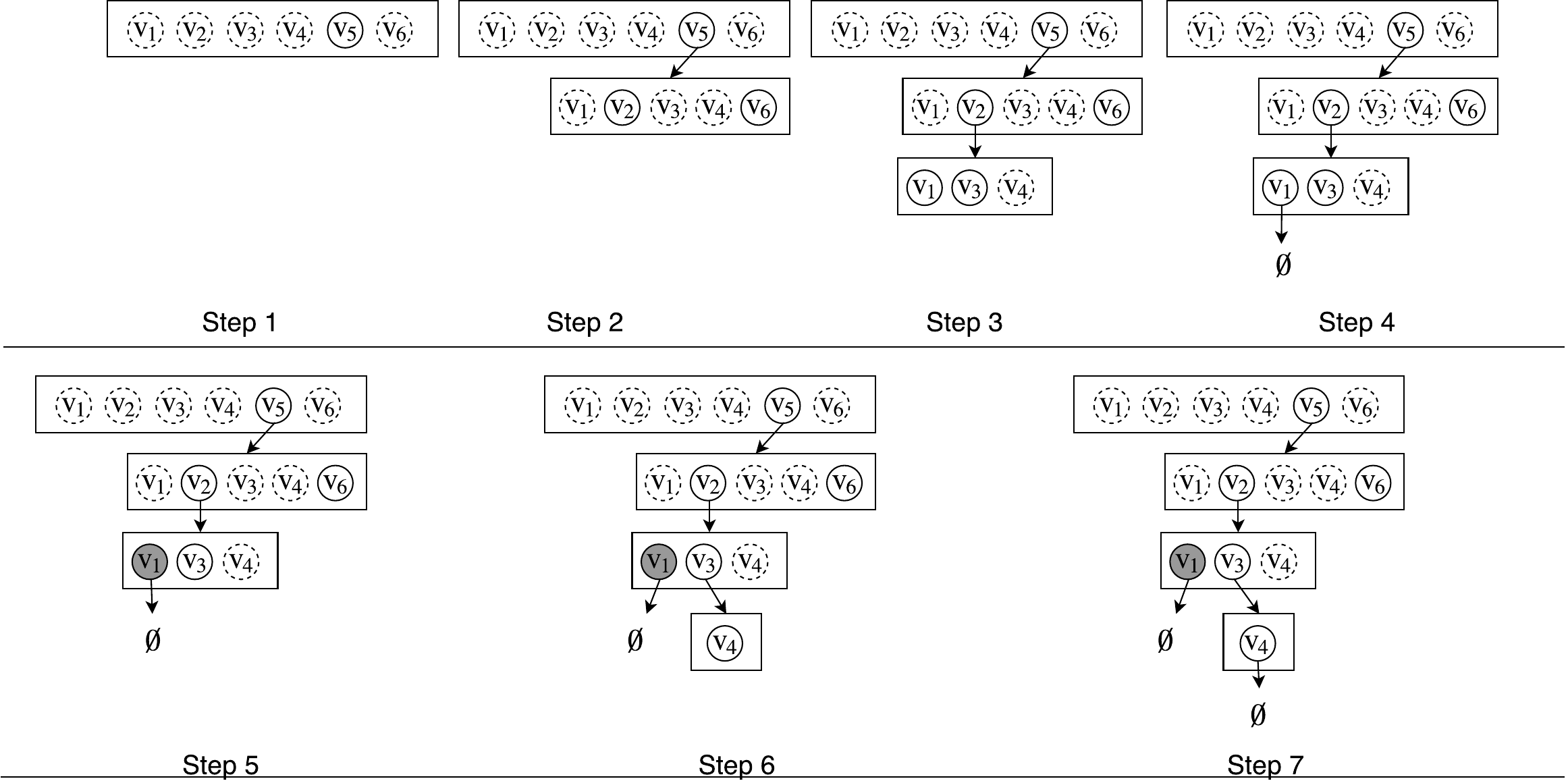}
		\put (0,46.5)  {\small Initial $S$}
		\put (0,43)  {\small Initial $C=\{v_5\}$}
		\put (0,40)  {\small Initial $F=\emptyset$}

		\put (32,41.5)  {\small $S^{\prime}$}
		\put (32,38)  {\small $C^{\prime}=\{v_2, v_6\}$}
		\put (32,35)  {\small $F^{\prime}=\emptyset$}

		\put (54,36)  {\small $S^{\prime\prime}$}
		\put (54,33.5)  {\small $C^{\prime\prime}=\{v_1, v_3\}$}
		\put (54,31)  {\small $F^{\prime\prime}=\emptyset$}
		\put (78,31)  {\small $S^{\prime\prime\prime}$}

		\put (5,5)  {\small $F^{\prime\prime} \leftarrow \{v_1\}$}

		\put (84.2,31) {\small $R_{best} \leftarrow \{v_5, v_2, v_1\}$}

		\put (81.5,3) {\small $R_{best} \leftarrow \{v_5, v_2, v_3, v_4\}$}
	\end{overpic}
	\caption{Progression of the search tree of Algorithm~\ref{alg:expand2} when solving MC for the example graph in Fig.~\ref{fig:mcexample}.  Vertices are visited in same order that in Fig.~\ref{fig:tree}, however only vertices in $C\subseteq S$ are expanded (continue circles). The first 7 steps are  incrementally generated to show vertex expansions and vertex deletions (grey vertices). Leaves of the tree correspond to the empty set.}
	\label{fig:mcxtree}
\end{figure*}

To more intuitively explain the proposed pruning step, we again use the sample input graph in Fig.~\ref{fig:mcexample}. The first 7 steps of Algorithm~\ref{alg:expand2} until the optimal solution is reached are shown in Fig.~\ref{fig:mcxtree} (contrast with the 12 steps needed for Algorithm~\ref{alg:simple}; see Fig.~\ref{fig:tree}). Vertices were expanded in the same order as in Algorithm~\ref{alg:simple} (Line~\ref{alg:mcx:initsort} in Algorithm~\ref{alg:expand2} was set to reorder vertices such that vertices are expanded in the order $v_1, v_2, \ldots, v_6$).

For each node $S$ in the search tree of Algorithm~\ref{alg:expand2}, we aim to find a subset $C \subseteq S $, such that only vertices in $C$ need to be expanded to find the optimal solution. In Fig.~\ref{fig:mcxtree}, $C=\{v_5\}$, $C^{\prime}=\{v_2, v_6\}$ and $C^{\prime\prime}=\{v_1,v_3\}$.

Assume momentarily that $S$ does not change during the progression of Algorithm~\ref{alg:expand2}, i.e., no vertices are removed after expansions. Also assume that for a vertex $v_k$ in $S$, all larger cliques  expanded from $v_k$ have already been explored. If an unexpanded vertex $v_i$ in $S$ is adjacent to $v_k$, any clique containing $R \cup \{v_i\}$ must have already been reached through expanding $v_k$. Thus, there is no need to expand $v_i$. This observation explains the core insight used by the proposed pruning.

Similar to Algorithm~\ref{alg:simple} and Algorithm~\ref{alg:expand}, Algorithm~\ref{alg:expand2} removes a vertex from $S$ after exploring its child nodes. To choose $v_k$ for a node $S$, Algorithm~\ref{alg:expand2} keeps record of removed vertices by adding them to a set $F$ associated to $S$ - e.g., in Step 5 in Fig.~\ref{fig:mcxtree}, $v_1$ is added to $F^{\prime\prime}$ after it is removed from $S^{\prime\prime}$. In the root node of the search tree, $F$ is initiated as $\emptyset$. For a child node $S^\prime$, the corresponding set of removed vertices is initiated as
\begin{align}
F^\prime = S \cap F.
\end{align}

$C$ can be obtained before expanding any vertex in $S$. Note that $v_k$ must be in $C$ as only nodes in $S$ that are adjacent to $v_k$ will not be expanded ($v_k \notin \Gamma(v_k)$). These observations are formalised in the following theorem.
\begin{theorem}\label{theo:additional}
	For  a vertex $v_k$ in $S \cup F$, the largest clique $Q$ containing $R$ must contain at least one vertex in $S\setminus \Gamma(v_k)$.
\end{theorem}
\begin{proof}
	The proof is by contradiction. Consider that the largest clique  containing $R\cup \{v_k\}$ has been already obtained. Then, $Q$ can be found by expanding vertices in $S \setminus \{v_k\}$. If  $Q = (R \cup P) \cap (S  \setminus \{v_k\})$ with $P \subseteq S  \cap \Gamma(v_k)$, then $Q \cup \{ v_k \}$ is a larger clique that contains $R$. However, this contradicts $Q$ being the largest clique.
\end{proof}
From Theorem~\ref{theo:additional}, we can take  $C := S\setminus\Gamma(v_k)$, for any $v_k$ in $S \cup F$. We choose $v_k$ such that we can skip expanding as many vertices as possible
	\begin{align}\label{eq:skip}
	\begin{aligned}
	&\underset{v_k \,\in\, S\cup F }{\text{maximise}}
	& & \left| S \cap \Gamma(v_k) \right|.
	\end{aligned}
	\end{align}
Solving~\eqref{eq:skip} is equivalent to find the vertex with the highest degree in the subgraph with vertices in $S$,  which can be solved in quadratic time. Since~\eqref{eq:skip} is solved only once in each node of the search tree, this extra step  introduces minor computation cost. Experiments in Sec.~\ref{sec:results} show that this extra step produces significant speed-ups when finding the optimal solution for problem~\eqref{eq:pw}.

Note that Algorithm~\ref{alg:expand2} also uses pruning with graph colouring. The colouring of a node $S$ is obtained using the greedy heuristic used in MCQ. However, removing vertices from $S$ may invalidate that colours are consecutive numbers, since not necessarily the last vertex in $S$ (associated with the largest colour) is removed if expanding vertices in $C$ only.  As colours are sorted, colour numbers can be reassigned in linear time (Line~\ref{alg:mcx:updatecol} in Algorithm~\ref{alg:expand2}).

\begin{table*}[t!]
	\centering
	\renewcommand{\arraystretch}{1.15}
	\resizebox{.99\textwidth}{!}{
		\begin{tabular}{|l|r|r|rr|rr|rrr|r|r|r|r|r|}
			\cline{1-15}
			\multicolumn{1}{|c|}{\multirow{2}{*}{Object}} &
			\multicolumn{1}{c|} {\multirow{2}{*}{$N$}} &
			\multicolumn{1}{c|} {\multirow{2}{*}{Outlier}} &
			\multicolumn{2}{c|} {\multirow{2}{*}{Consistency graph}} &
			\multicolumn{2}{c|} {\multirow{2}{*}{Inconsistency graph}} &
			\multicolumn{3}{c|} {\multirow{2}{*}{Solution}} &
			\multicolumn{1}{c|} {\multirow{2}{*}{\mcx{}}} &
			\multicolumn{1}{c|} {\multirow{2}{*}{MCQ}} &
			\multicolumn{1}{c|} {\multirow{2}{*}{MIP-MC}} &
			\multicolumn{1}{c|} {\multirow{2}{*}{MIP-VC}} &
			\multicolumn{1}{c|} {\multirow{2}{*}{RANSAC}}
			\\[.6em]
			& 
			& \multicolumn{1}{c|}{ratio} 
			& \multicolumn{1}{c}{$|V|$}     
			& \multicolumn{1}{c|}{$|E|$}
			& \multicolumn{1}{c}{$|V|$}     
			& \multicolumn{1}{c|}{$|E|$}
			& \multicolumn{1}{c}{$|I^*|$}   
			& \multicolumn{1}{c}{angErr (\textdegree)}
			& \multicolumn{1}{c|}{trErr}
			& \multicolumn{1}{c|}{time (s)}     
			& \multicolumn{1}{c|}{time (s)}     
			& \multicolumn{1}{c|}{time (s)}     
			& \multicolumn{1}{c|}{time (s)}     
			& \multicolumn{1}{c|}{time (s)}     
			\\
			\hline
			\multirow{3}{*}{{\parbox[c][4em][c]{1.8cm} {\emph{armadillo}\\
						$|\cX|=788$\\
						$|\cY|=711$ }}}
			& 1000   & 0.98 & 1000   & 167200 & 1000   & 831800 & 37     & 0.72   & 0.06   & 0.020  & 0.010  & 1387.404 & 1291.205 & 25.25\\
			& 3000   & 0.98 &3000   & 1397630 & 3000   & 7599370 & 92     & 0.61   & 0.05   & 2.510  & 5.390  & --     & --     &39.93\\
			& 5000   & 0.98 &5000   & 3485008 & 5000   & 21509992 & 139    & 0.98   & 0.15   & 61.080 & 3377.430 & --     & --     &70.16\\
			\hline
			\multirow{3}{*}{{\parbox[c][4em][c]{1.8cm} {\emph{buddha}\\
			$|\cX|=206$\\
			$|\cY|=193$ }}}
			& 1000   & 0.96 &1000   & 139516 & 1000   & 859484 & 55     & 0.64   & 0.34   & 0.170  & 0.020  & 75.170 & 85.624 & 5.67\\
			& 3000   & 0.98 &3000   & 700794 & 3000   & 8296206 & 68     & 1.60   & 0.18   & 5.420  & 2.730  & --     & --     &105.24\\
			& 5000   & 0.99 &5000   & 1700434 & 5000   & 23294566 & 74     & 1.06   & 0.25   & 5.640  & 76.100 & --     & --     &602.04\\
			\hline
			\multirow{3}{*}{{\parbox[c][4em][c]{1.8cm} {\emph{bunny}\\
						$|\cX|=668$\\
						$|\cY|=615$ }}}
			& 1000   & 0.96 &1000   & 101956 & 1000   & 897044 & 27     & 2.03   & 0.42   & 0.020  & 0.010  & --     & --     &5.09\\
			& 3000   & 0.96 &3000   & 918160 & 3000   & 8078840 & 102    & 0.32   & 0.11   & 0.410  & 0.280  & 1224.189 & 1642.969 &7.32\\
			& 5000   & 0.96 &5000   & 2490092 & 5000   & 22504908 & 171    & 0.47   & 0.09   & 7.250  & --     & 3520.973 & --     &16.39\\
			\hline
			\multirow{3}{*}{{\parbox[c][4em][c]{1.8cm} {\emph{chef}\\
						$|\cX|=183$\\
						$|\cY|=185$ }}}
			& 1000   & 0.94 &1000   & 117182 & 1000   & 881818 & 80     & 1.01   & 0.21   & 0.310  & 0.040  & 16.469 & 22.378 & 1.32\\
			& 3000   & 0.97 &3000   & 774036 & 3000   & 8222964 & 97     & 0.95   & 0.16   & 1.980  & 1.410  & 1578.146 & 2293.338 &30.39\\
			& 5000   & 0.98 &5000   & 1961382 & 5000   & 23033618 & 100    & 0.61   & 0.26   & 6.870  & 37.470 & --     & --     &200.01\\
			\hline
			\multirow{3}{*}{{\parbox[c][4em][c]{1.8cm} {\emph{chicken}\\
						$|\cX|=601$\\
						$|\cY|=616$ }}}
			& 1000   & 0.97 &1000   & 142856 & 1000   & 856144 & 26     & 15.88  & 2.02   & 0.040  & 0.020  & --     & --     & 14.53\\
			& 3000   & 0.98 &3000   & 1265814 & 3000   & 7731186 & 55     & 1.63   & 0.28   & 1.480  & 0.830  & --     & --     &58.08\\
			& 5000   & 0.98 &5000   & 3597810 & 5000   & 21397190 & 81     & 0.93   & 0.23   & 7.620  & 21.420 & --     & --     &98.33\\
			\hline
			\multirow{3}{*}{{\parbox[c][4em][c]{1.8cm} {\emph{dragon}\\
						$|\cX|=289$\\
						$|\cY|=270$ }}}
			& 1000   & 0.91 &1000   & 141516 & 1000   & 857484 & 106    & 0.23   & 0.20   & 0.090  & 0.020  & 19.240 & 21.579 & 0.52\\
			& 3000   & 0.97 &3000   & 877756 & 3000   & 8119244 & 126    & 0.41   & 0.20   & 1.030  & 0.930  & 772.552 & 1294.577 &12.74\\
			& 5000   & 0.98 &5000   & 2211540 & 5000   & 22783460 & 136    & 0.31   & 0.19   & 5.530  & 18.130 & --     & --     &76.86\\
			\hline
			\multirow{3}{*}{{\parbox[c][4em][c]{1.8cm} {\emph{parasauro}\\
						$|\cX|=261$\\
						$|\cY|=216$ }}}
			& 1000   & 0.93 &1000   & 153806 & 1000   & 845194 & 81     & 0.14   & 0.10   & 0.040  & 0.020  & 19.053 & 24.703 &  1.26\\
			& 3000   & 0.97 &3000   & 973874 & 3000   & 8023126 & 118    & 0.40   & 0.10   & 2.830  & 42.160 & 2289.741 & 2681.053 &21.83\\
			& 5000   &  0.98 &5000   & 2214264 & 5000   & 22780736 & 127    & 0.44   & 0.22   & 36.600 & --     & --      & --      &84.86\\
			\hline
			\multirow{3}{*}{{\parbox[c][4em][c]{1.8cm} {\emph{t-rex}\\
						$|\cX|=222$\\
						$|\cY|=217$ }}}
			& 1000   & 0.93 &1000   & 116406 & 1000   & 882594 & 86     & 0.43   & 0.15   & 0.040  & 0.010  & 13.601 & 15.705 & 0.88\\
			& 3000   & 0.97 &3000   & 818628 & 3000   & 8178372 & 118    & 0.13   & 0.21   & 1.200  & 10.570 & 865.289 & 1339.081 &18.15\\
			& 5000   & 0.98 &5000   & 2022928 & 5000   & 22972072 & 128    & 0.27   & 0.22   & 7.970  & 1287.740 & --     & --     & 81.86\\
			\hline
		\end{tabular}
	} 
	\caption{Results for matching with pairwise constraints~\eqref{eq:pw} and RANSAC. `--' implies forced timeout after a 1-hour limit.}
	\label{tab:pcr}
\end{table*}

\begin{figure*}
	\def\figh{1.42cm}
	\begin{tabularx}{\textwidth}{c|C{.3}|C{.3}|C{.3}}
		& {\small Input correspondence set}  & {\small Pairwise consistent correspondences} & {\small Alignment}\\
		\hline
		&&&\\[-1em]
		\rotatebox[origin=l]{90}{\small\emph{armadillo}}&
		\includegraphics[height=\figh]{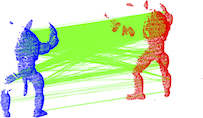}&
		\includegraphics[height=\figh]{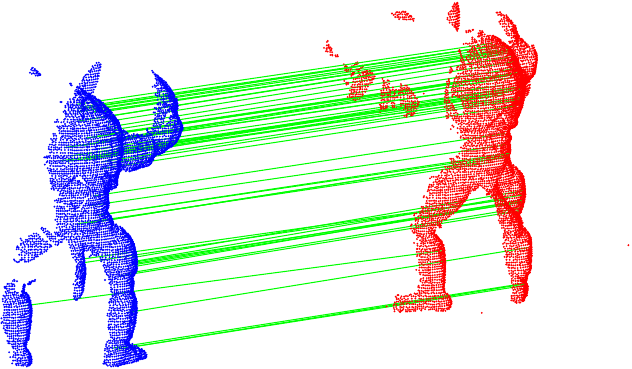}&
		\includegraphics[height=\figh]{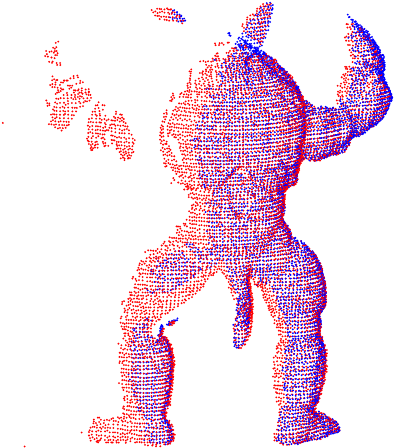}
		\\[-.2em]
		\rotatebox[origin=l]{90}{\small\emph{buddha}}&
		\includegraphics[height=\figh]{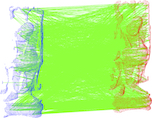}&
		\includegraphics[height=\figh]{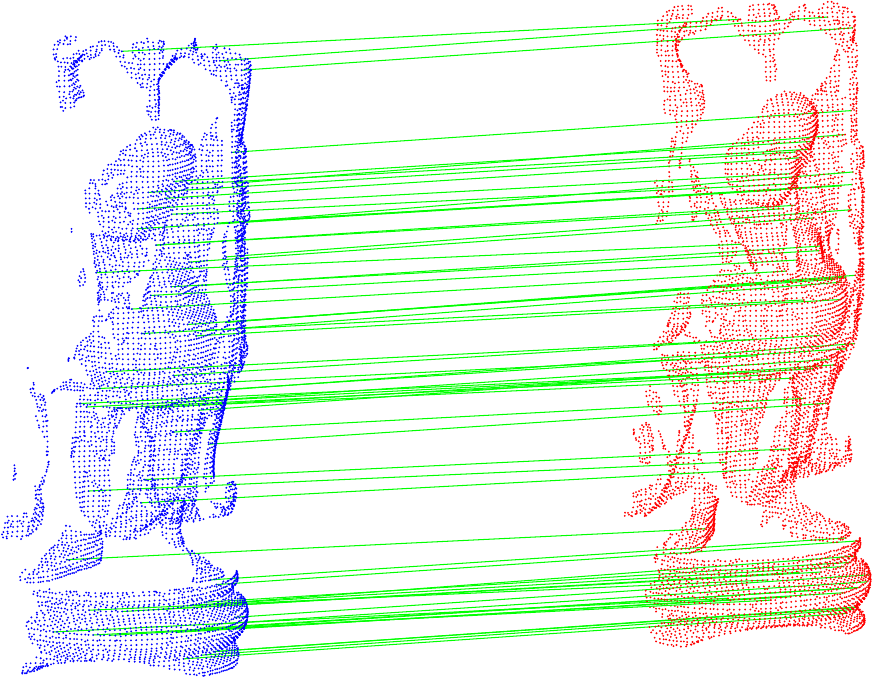}&
		\includegraphics[height=\figh]{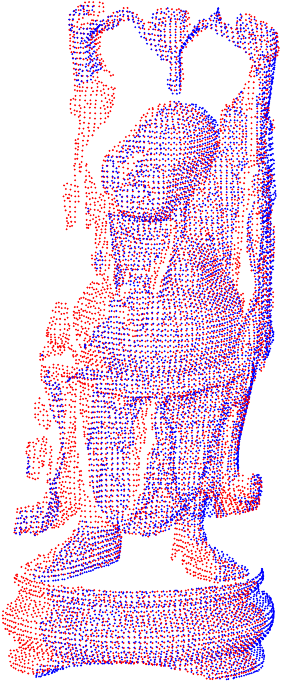}
		\\[-.2em]
		\rotatebox[origin=l]{90}{\small\emph{bunny}}&
		\includegraphics[height=\figh]{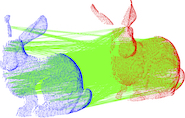} &
		\includegraphics[height=\figh]{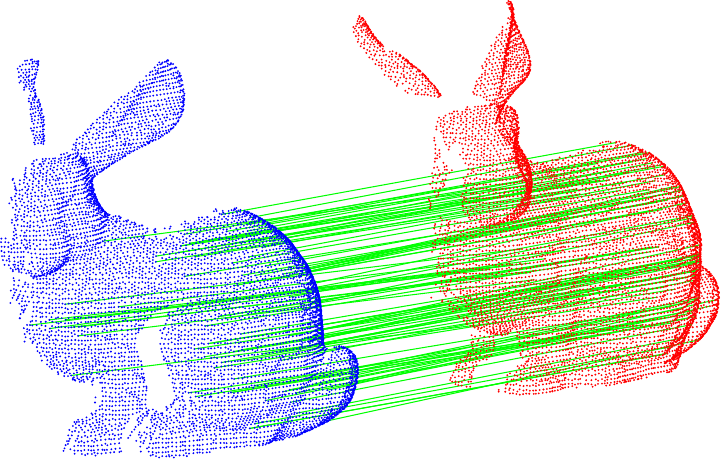} &
		\includegraphics[height=\figh]{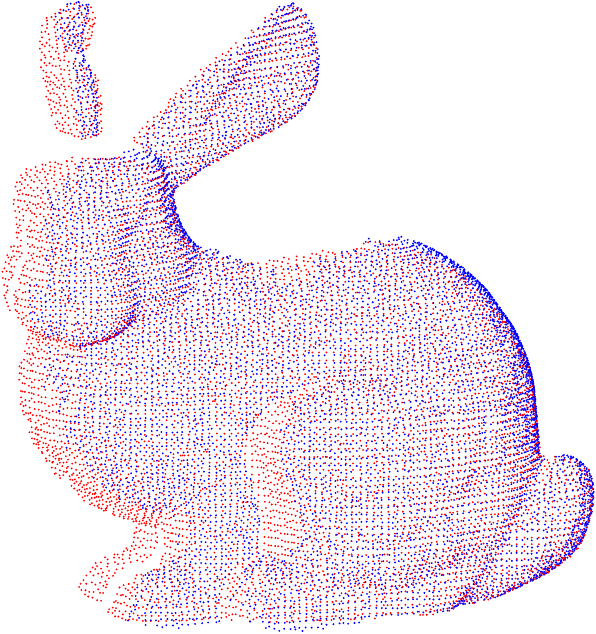}
		\\[-.2em]
		\rotatebox[origin=l]{90}{\small\emph{chef}}&
		\includegraphics[height=\figh]{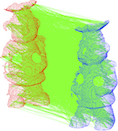}&
		\includegraphics[height=\figh]{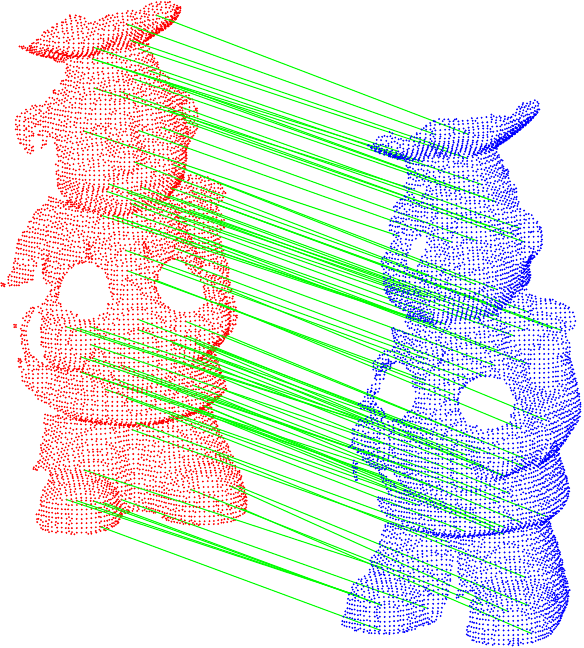}&
		\includegraphics[height=\figh]{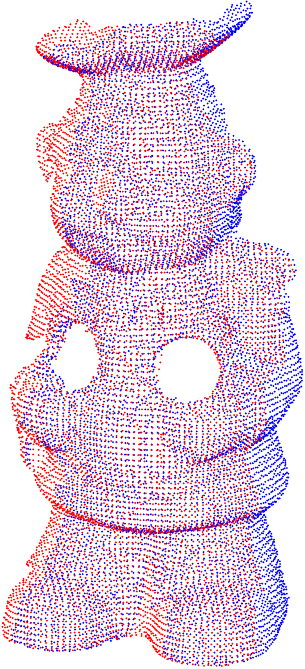}
		\\[-.2em]
		\rotatebox[origin=l]{90}{\small\emph{chicken}}&
		\includegraphics[height=\figh]{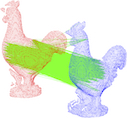}&
		\includegraphics[height=\figh]{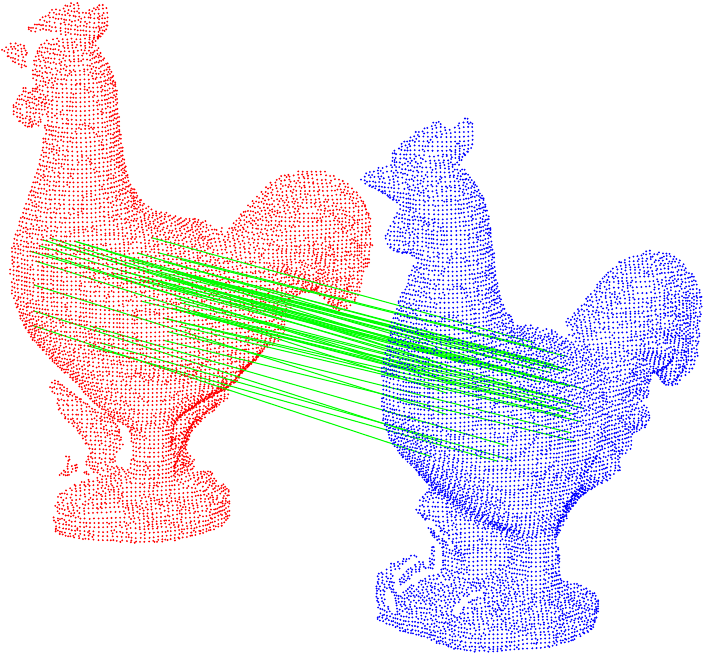}&
		\includegraphics[height=\figh]{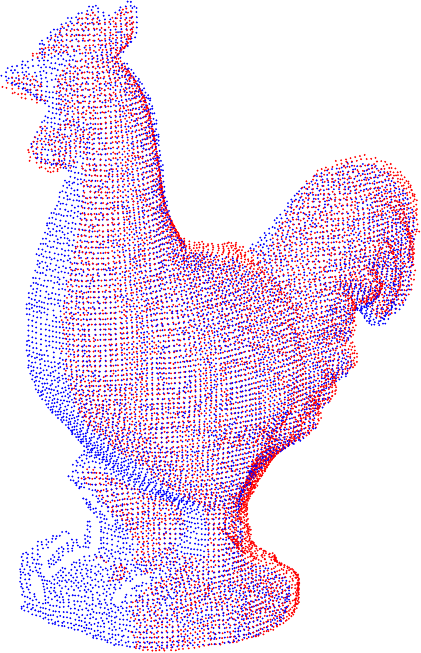}
		\\[-.2em]
		\rotatebox[origin=l]{90}{\small\emph{dragon}}&
		\includegraphics[height=\figh]{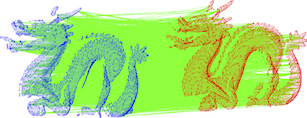}&
		\includegraphics[height=\figh]{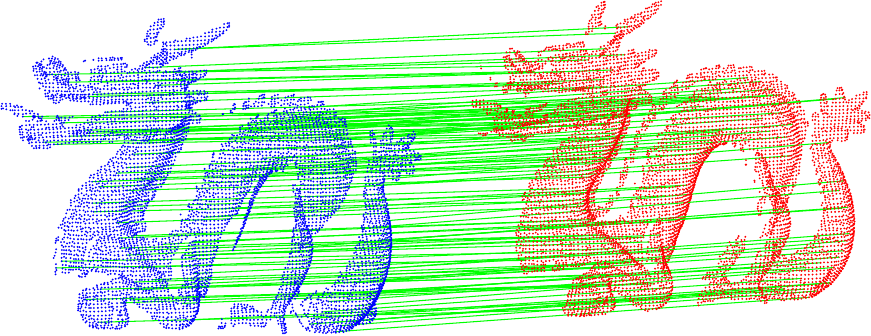}&
		\includegraphics[height=\figh]{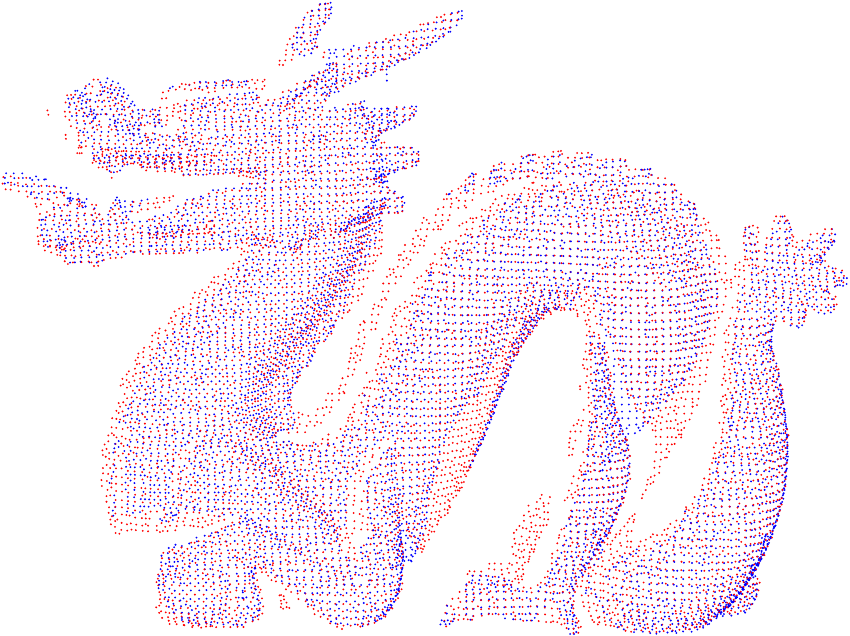}
		\\[-.2em]
		\rotatebox[origin=l]{90}{\small\emph{parasauro}}&
		\includegraphics[height=\figh]{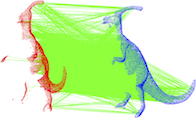}&
		\includegraphics[height=\figh]{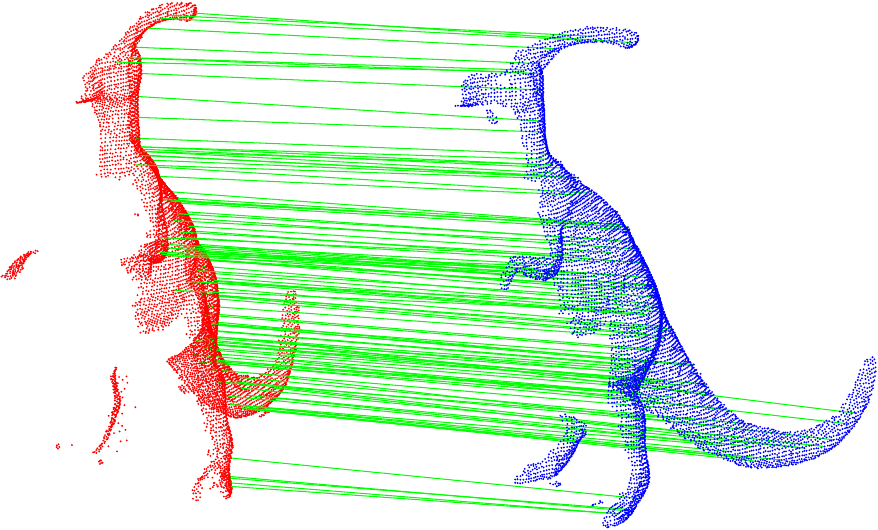}&
		\includegraphics[height=\figh]{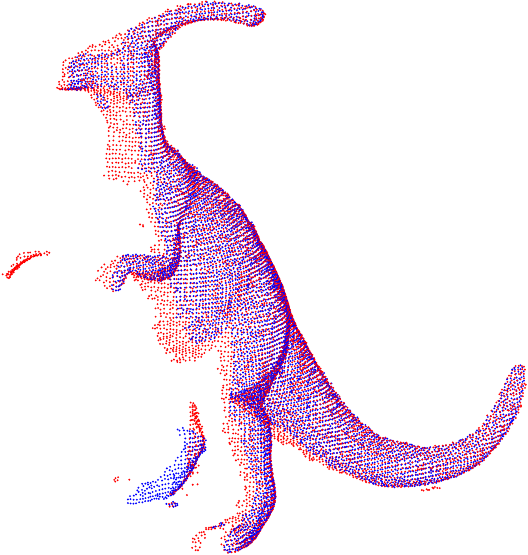}
		\\[-.2em]
		\rotatebox[origin=l]{90}{\small\emph{t-rex}}&
		\includegraphics[height=\figh]{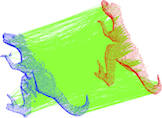}&
		\includegraphics[height=\figh]{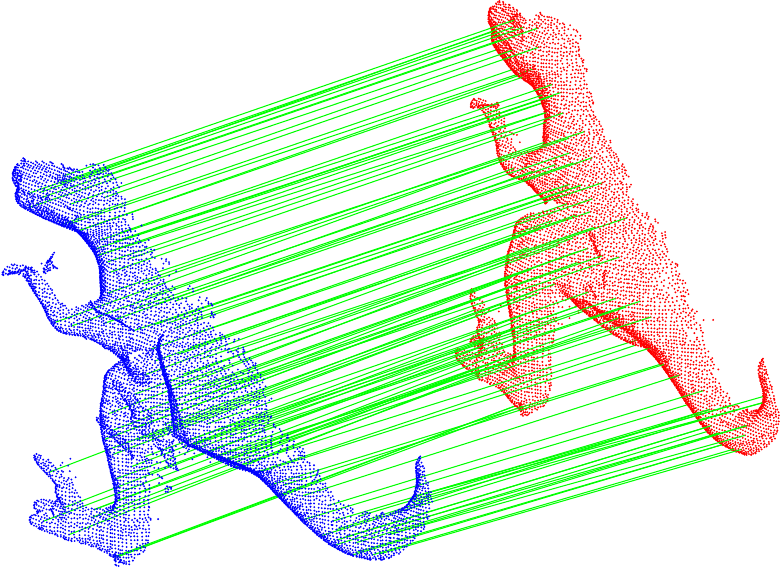}&
		\includegraphics[height=\figh]{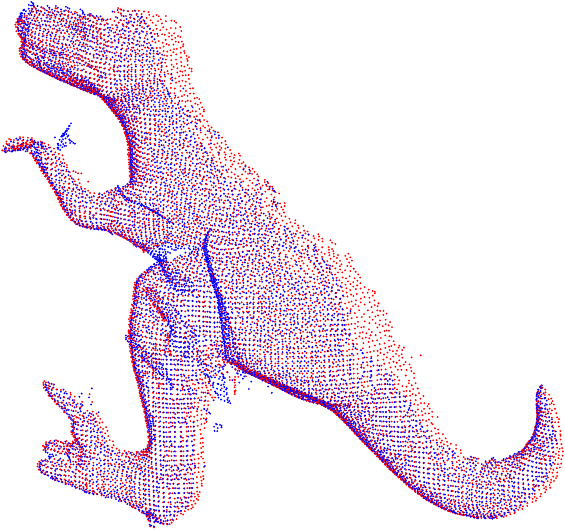}
		\\[-.2em]
	\end{tabularx}
	\caption{Qualitative results. Column 1: Input correspondence set $\cC$ with $N=3000$. Column 2: Largest subset of pairwise consistent correspondences $\cI^\ast$. Column 3: The alignment given by the rigid transformation estimated from $\cI^\ast$ using SVD~\cite{horn87}.}
	\label{fig:pcr}
\end{figure*}

\section{Results}\label{sec:results}

To investigate the efficacy of the proposed method, we compared the following algorithms for matching with pairwise constraints~\eqref{eq:pw}:
\begin{itemize}[itemsep=0.0pt,topsep=0.5pt,parsep=0.5pt]
	\item \textbf{MCQ}: BnB method (Algorithm~\ref{alg:expand}) presented in~\cite{tomita03} \hl{(using the C++ implementation of}~\cite{eppstein11}\footnote{\url{https://github.com/darrenstrash/quick-cliques}}).
	\item \textbf{PMC}: Our BnB method (Algorithm~\ref{alg:expand2}) with the same initial vertex ordering as MCQ.
	\item \textbf{MIP-MC}: MIP formulation of MC solved with Gurobi Optimiser.
	\item \textbf{MIP-MVC}: MIP formulation of MVC solved with Gurobi Optimiser.
\end{itemize}
PMC was implemented in C++, and all experiments were conducted on a standard PC with a 2.50 GHz CPU and 8 GB of RAM. Note that the currently state-of-the-art algorithm for matching with pairwise constraints in computer vision is~\cite{enqvist09}. Our experimentations suggest however that Enqvist et al.~(which is an MVC solver\footnote{Based on our own implementation. The original authors' implementation was unavailable.}) is outperformed generally by MIP-MVC, e.g., the former was unable to terminate within the imposed time limit on all data instances. We thus do not include Enqvist et al.~in our benchmark. 


We tested all methods on scans of objects from two different datasets, namely, the Stanford 3D Scanning Repository~\cite{curless96} (\emph{armadillo}, \emph{buddha}, \emph{bunny}, and \emph{dragon}), and Mian's dataset~\cite{mian06} (\emph{chef}, \emph{chicken}, \emph{parasauro} and \emph{t-rex}). The scans of the objects are illustrated in Fig.~\ref{fig:pcr}.

Two partially overlapping scans $\cV_1$ and $\cV_2$ were selected for each object. $\cV_1$ and $\cV_2$ were centred and scaled such that both point sets were contained in the cube $[-50,50]^3$. The set of point correspondences $\cC$ was obtained from the sets of keypoints $\cX$ and $\cY$  using the state-of-the-art ISS3D~\cite{zhong09} detector and the PFH~\cite{rusu08} descriptor, as implemented on Point Cloud Library\footnote{\url{http://pointclouds.org/}}. The sizes of $\cX$ and $\cY$ are listed on the Column 1 in Table~\ref{tab:pcr}.



The $N$ best keypoint correspondences (sorted based on the $\ell_2$-norm of the PFH descriptors), where $N \in \{1000,3000,5000\}$, were retained to create instances of the pairwise maximisation problem~\eqref{eq:pw} for point cloud registration. Based on the graph construction methods outlined in Sec.~\ref{sec:graph}, the size of the consistency and the inconsistency graphs are listed in Columns 4--7 in Table~\ref{tab:pcr}. The inlier threshold $\epsilon$ for \eqref{eq:pw} was taken as twice the average nearest neighbour distance in $\cV_1$ and $\cV_2$.

To give an indication of the ratio of correct correspondences generated for each $\cV_1$ and $\cV_2$ pair, we registered the point clouds using the ground truth rigid transformation and determined the subset of $\cC$ (the outliers) that cannot be aligned up to $\epsilon/2$ (using $\epsilon/2$ ensures the equivalence of the thresholds for pairwise consistent alignment~\eqref{eq:pw} and ``unary" alignment; see~\cite{enqvist09} for details).
The outlier ratio for each input instance is listed in Column~3 in Table~\ref{tab:pcr}. Notice that the outlier rates are extremely high (up to $99\%$).


We recorded the following measures for each method:
\begin{itemize}[itemsep=0.0pt,topsep=0.5pt,parsep=0.5pt]
		\item $|I^*|$: The optimal value of~\eqref{eq:pw}.
		\item time (s): runtime of solving~\eqref{eq:pw}, i.e., after generating the input graph. A timeout of $1$ hour ($3600$ seconds) was imposed on all methods.
\end{itemize}
If a method could not terminate successfully within the time limit, the result was marked with a `-' in the table. Note that since all the methods return the global solution to~\eqref{eq:pw}, they have the same solution quality and differ only in runtime. To assess the goodness of the registration on each pair of point clouds, a rigid transformation was estimated using SVD~\cite{horn87} on $\cI^\ast$. The quality of the rigid transformation parameters was measured based on the following errors:
\begin{itemize}[itemsep=0pt,topsep=0.5pt,parsep=0.5pt]
		\item angErr (\textdegree): Angular error of estimated rotation w.r.t.~the true rotation.
		\item trErr: Translation error of estimated translation w.r.t.~the true translation.
\end{itemize}
Columns 8--14 in Table~\ref{tab:pcr} show the recorded measures. 

Although in general MIP-MC reported lower runtimes than MIP-MVC, MIP-MC still is impractical, i.e., it was unable to terminate within 1 hour on half of the generated input instances. On the instances where it was able to finish within the time limit, MIP-MC took several orders of magnitude more time than PMC to find the optimal solution.

In general, \mcx{} found the optimal solution in less than 10 seconds. Although MCQ performed better than \mcx{} for $N=1000$, \mcx{} converged considerably faster than MCQ for $N=5000$ - e.g., the runtime of MCQ was more than 1000 seconds in \emph{armadillo} and \emph{t-rex}. Note that MCQ did not converge within the time limit for all instances. Since MCQ and PMC use the same greedy heuristic for colouring, the observed runtimes suggest that the proposed additional pruning step in Sec.~\ref{sec:bnbextra} played an important role for solving~\eqref{eq:pw} on large graphs. Observe that for $N=5000$, the number of vertices of the consistency graph  is one order of magnitude larger than for $N=1000$. Note that we are more likely to obtain good registration from larger correspondence sets. Observe that for $N=1000$, the optimal solution of~\eqref{eq:pw} did not produce a good alignment in all cases; in particular, the angular error was considerable for \emph{bunny} and \emph{chicken}. Thus, the performance gain of PMC over MCQ on large graphs is of high practical significance. 

Fig.~\ref{fig:pcr} depicts qualitative results for all the tested objects with $N=3000$ correspondences. As alluded above, the rigid transformations were estimated using SVD from $\cI^\ast$.

\vspace{-1em}
\paragraph{Comparison against RANSAC} As a baseline, on each input instance $\cC$, we executed RANSAC~\cite{fischler81} to find the largest subset of correspondences that are ``unary" consistent, i.e., the largest consensus set. We used a confidence level $0.99$ for the stopping criterion. Column~15 in Table~\ref{tab:pcr} shows the median runtime of RANSAC over 100 runs. Observe that the runtime of RANSAC is quite significant, i.e., up to one order of magnitude greater than the runtime of PMC. This is because the runtime of RANSAC grows exponentially with the outlier ratio, and all the instances in Table~\ref{tab:pcr} have more than $90\%$ outliers.


\section{Conclusions}

We have shown that matching with pairwise constraints can be performed in reasonable time when posing the problem as maximum clique. We have also proposed a maximum clique algorithm that combines graph colouring with a proposed extra pruning step to very efficiently solve maximum clique. The obtained results demonstrate that, using the proposed algorithm, matching with pairwise constraints is a very practical approach for point cloud registration and an excellent alternative to standard robust estimation.


\clearpage

{\small
	\bibliographystyle{ieee}
	\bibliography{ms}
}

\end{document}